\documentclass[a4paper,12pt]{article}
\usepackage{amsmath}
\usepackage{algorithmic}
\usepackage{algorithm}
\usepackage{amsthm}
\usepackage{amsfonts}
\usepackage{hyperref}
\newtheorem{theorem} {Theorem}
\newtheorem{lemma} {Lemma}
\newtheorem{claim} {Claim}
\newtheorem{corr} {Corollary}
\usepackage{times}
\newcommand{\ignore}[1]{}

\title{Universal MMSE Filtering With Logarithmic Adaptive Regret}

\author{Dan Garber\\  
\small{Technion - Israel Inst. of Tech.} \\
\small{dangar@cs.technion.ac.il} 
\and Elad Hazan \\
\small{Technion - Israel Inst. of Tech.} \\ 
\small{ehazan@ie.technion.ac.il}}

\begin{document}
\maketitle

\begin{abstract}
We consider the problem of online estimation of a real-valued signal corrupted by oblivious zero-mean noise using  linear estimators. The estimator is required to iteratively predict the underlying signal  based on the current and several last noisy observations, and its performance is measured by the mean-square-error. We describe and analyze an algorithm for this task which: 
\begin{enumerate}
\item
Achieves logarithmic adaptive regret against the best linear filter in hindsight. This bound is assyptotically tight, and resolves the question of Moon and Weissman \cite{MWPaper}. 
\item
Runs in linear time in terms of the number of filter coefficients. Previous constructions required at least quadratic time. 
\end{enumerate}
\end{abstract}

\section{Introduction}
We consider the problem of filtering: designing algorithms for the causal estimation of a real valued signal from noisy observations. The filtering algorithm observes at each iteration a noisy signal component, and is required to estimate the corresponding underlying signal component based on the current and past noisy observations alone.
\\
We consider finite fixed-length linear filters that combine the current and several last noisy observations for prediction of the current  underlying signal component.  Performance is measured by the  mean square error over the entire signal. Following the setting in \cite{MWPaper}, we assume that the underlying signal is an arbitrary bounded signal, possibly even adversarial, and that it is corrupted by an additive zero-mean, time-independent, bounded noise with known constant variance \footnote{The justification of \cite{MWPaper} for assuming that the variance is a known constant is that this variance could be learned by sending a training sequence in the beginning of transmission.}. 
\\
The approach taken in this paper is to construct a {\it universal} filter - i.e. an adaptive filter whose performance we compare to an optimal offline filter with full knowledge of the signal and noise. The metric of performance is thus regret - or the difference between the total mean squared error incurred by our adaptive filter, and the total mean square error of the offline benchmark filter. 
\\
The question of competing with a fixed offline filter was successfully tackled in \cite{MWPaper}. In this paper we consider a more challenging task: competing with the best offline changing filter, where restrictions are placed on how often this optimal offline filter is allowed to change. A more stringent metric of performance what fully captures this notion of competing with an adaptive offline benchmark is called {\it adaptive regret}: it is the maximum regret incurred by the algorithm on any subinterval.  
\\
We present and analyze simple, efficient and intuitive algorithms that attain logarithmic adaptive regret. This bound is tight, and resolves a question posed by Moon and Weissman in \cite{MWPaper}. Along the way, we introduce a simple universal algorithm for filtering, improving the previously known best running time from quadratic in the number of filter coefficients to linear.

\subsection{Related Work}

There has been much work on the problem of estimating a real-valued signal from noisy observations with respect to the MMSE loss over the years. Classical results assume a model in which the underlying signal is stochastic with some known parameters, i.e.  the first and second moments, or require the signal to be stationary, such as the classical work of \cite{WienerPaper}. The special case of linear MMSE filters has received special  attention  due to its simplicity \cite{LinearEstimationRef}. For more recent results on MMSE estimation see  \cite{RobustMMRef1, RobustMMRef2, RobustMMRef3,UnsuperAdapFilterRef}. 
\\
In this work we follow the non-stochastic setting of \cite{MWPaper}: no generating model is assumed for the underlying signal and stochastic assumptions are made on the added noise (that it is zero-mean, time-independent with known fixed variance). In this setting, while considering finite linear filters, \cite{MWPaper} presented an online algorithm that achieves logarithmic expected regret with respect to the entire signal. The computational complexity of their algorithm is proportional to a quadratic in the linear filter size. 
\\
Henceforth we build on recent results from the emerging online learning framework called  online convex optimization  \cite{ZinkPaper,LogRegretPaper}. For our adaptive regret algorithm, we use tools from the framework presented in \cite{AdapRegretPaper} to derive an algorithm that achieves logarithmic expected regret on any interval of the signal.

\section{Preliminaries}

\subsection{Online convex optimization} 
In the setting of online convex optimization (OCO) an online algorithm $\mathcal{A}$ is iteritevly required to make a prediction by choosing a point $x_t$ in some convex set $\mathcal{K}$. The algorithm then incurs a loss $l_t(x_t)$, where $l_t(x):\mathcal{K}\rightarrow\mathbb{R}$ is a convex function. The emphasis in this model is that on iteration $t$, $\mathcal{A}$ has only knowledge of the loss functions in previous iterations $l_1(x),...,l_{t-1}(x)$ and thus $l_t(x)$ may be chosen arbitrarily and even adversely. The standard goal in this setting is to minimize the difference between the overall loss of $\mathcal{A}$ and that of the best fixed point $x^*\in{\mathcal{K}}$ in hindsight. This difference is called regret and it is formally given by,
\begin{eqnarray*}
R_T(\mathcal{A}) = \sum_{t=1}^Tl_t(x_t) - \min_{x\in{\mathcal{K}}}\sum_{t=1}^Tl_t(x) 
\end{eqnarray*}
A stronger measure of performance requires the algorithm to have little regret on any interval $I=[r,s]\subseteq{[T]}$ with respect to the best fixed point $x_I^*\in{\mathcal{K}}$ in hindsight in this interval. This measure is call adaptive regret and it is given by  ,
\begin{eqnarray*}
AR_T(\mathcal{A}) = \sup_{I = [r,s]\subset{[T]}}\lbrace{\sum_{t=r}^sl_t(x_t) - \min_{x\in{\mathcal{K}}}\sum_{t=r}^s(l_t(x)}\rbrace
\end{eqnarray*}
\subsection{Problem Setting}
Let $x_t$ be a real-valued, possibly adversarial, signal bounded in the range $[-B_X...B_X]$. The signal $x_t$ is corrupted by an additive zero-mean time independent noise $n_t$ bounded in the range $[-B_N...B_N]$ with known time-invariant variance $\sigma^2$. An estimator observes on time $t$ the noisy signal $y_t = x_t + n_t$, and is required to predict $x_t$ by taking a linear combination of the observations $y_t,y_{t-1},...,y_{t-d+1}$ where $d$ is the order of the filter. That is, the estimator chooses on time $t$ a filter $w_t\in{\mathbb{R}^d}$ and predicts according to $w_t^{\top}Y_t$ where $Y_t\in{\mathbb{R}^d}$ and $Y_t(i) = y_{t-i+1}$, $1\leq i \leq d$. The loss of the estimator after $T$ iterations is given by the mean-square-error $\frac{1}{T}\sum_{t=1}^T(x_t - w_t^{\top}Y_t)^2$. \\
In case $x_t$ is observable to the online algorithm, minimizing the regret and the adaptive regret is fairly easy using the framework of OCO with the loss functions $l_t(w_t) = (x_t -w_t^{\top}Y_t)^2$. However in our case, the algorithm only observes the noisy signal $y_t$ and thus online convex optimization algorithms could be directly used. Denoting $\hat{l}_t(w) = (y_t - w^{\top}Y_t)^2 + 2w^{\top}c$ where $c\in{\mathbb{R}^d}$, $c = (\sigma^2,0...,0)$, it was pointed out in \cite{MWPaper} that if $w_t$ depends only on the observations $y_1,...,y_{t-1}$, then for any $w\in{\mathbb{R}^d}$ it holds that, 
\begin{eqnarray}\label{ExpectationEq}
\mathbb{E}\left[{\sum_{t=1}^T\hat{l}_t(w_t) - \sum_{t=1}^T\hat{l}_t(w)}\right] = \mathbb{E}\left[{\sum_{t=1}^Tl_t(w_t) - \sum_{t=1}^Tl_t(w)}\right]
\end{eqnarray}
Thus by using OCO algorithms with the estimated loss functions $\hat{l}_t(w)$ we may minimize the expected regret with respect to the actual losses $l_t(w)$. Thus a simple algorithm such as \cite{ZinkPaper} immediately gives a $O(\sqrt{T})$ bound on the expected regret as well as on the expected adaptive regret with respect to the true losses $l_t(w)$, as long as we limit the choice of the filter to a euclidean ball of constant radius.

\subsection{Using Strong-Convexity and Exp-Concavity}
Given a function $f(x): \mathcal{K}\rightarrow\mathbb{R}$ we denote by $\nabla{}f(x)$ the gradient vector of $f$ at point $x$ and by $\nabla^2f(x)$ the matrix of second derivatives, also known as the Hessian, of $f$ at point $x$. $f(x)$ is convex at point $x$ if and only if $\nabla^2f(x) \succeq{0}$, that is its Hessian is positive semidefinite at $x$.\\ 
We say that $f$ is {\it $H$-strongly-convex}, for some $H>0$, if for all $x\in{\mathcal{K}}$ it holds that $\nabla^2f(x) \succeq{}H\textbf{I}$, where $\textbf{I}$ is the identity matrix of proper dimension. That is all the eigenvalues of $\nabla^2f(x)$ are lower bounded by $H$ for all $x\in{\mathcal{K}}$. \\
We say that $f$ is {\it $\alpha$-exp-concave}, for some $\alpha > 0$, if the function $\exp{(-\alpha{}f(x))}$ is a concave function of $x\in{\mathcal{K}}$. It is easy to show that given a function $f$ such that $f\succeq{}H\textbf{I}$ and $\max_{x\in{\mathcal{K}}}\Vert{\nabla{}f(x)}\Vert_2 \leq G$ it holds that $f$ is $\frac{H}{G^2}$-exp-concave.\\
In case all loss functions are $H$-strongly-convex or $\alpha$-exp-concave, there exists algorithms that achieve logarithmic regret and adaptive regret \cite{LogRegretPaper, AdapRegretPaper}. \\
In our case, the Hessian of the loss function $\hat{l}_t(w)$ is given by the random matrix $\nabla^2\hat{l}_t(w) = 2Y_tY_t^{\top}$ which is positive semidefinite and it holds that
\begin{eqnarray}\label{ExpectationOfMatrix}
\mathbb{E}\left[{Y_tY_t^{\top}}\right] = \mathbb{E}\left[{X_tX_t^{\top} + N_tX_t^{\top} + X_tN_t^{\top} + N_tN_t^{\top}}\right]
= X_tX_t^{\top} + \sigma^2\textbf{I} \succeq \sigma^2\textbf{I}
\end{eqnarray}
Nevertheless, in worst case, $\hat{l}_t(w)$ need not be strongly-convex or exp-concave and thus algorithms such as \cite{LogRegretPaper, AdapRegretPaper} could not be directly used in order to get logarithmic expected regret and adaptive regret. 

\section{A Simple Gradient Decent Filter}
In this section we describe how the problem of the loss functions $\hat{l}_t$ not necessarily being strongly-convex or exp-concave could be overcome and introduce a simple gradient decent algorithm based on \cite{LogRegretPaper} that achieves $O(\log{T})$ expected regret. \\
For time $t$ and filter $w\in{\mathbb{R}^d}$ we define the following loss functions.

\begin{eqnarray}\label{NewLossFunc}
L^k_t(w) = \sum_{\tau = t-k+1}^t\hat{l}_t(w) + (w-w_t)^{\top}\left({(k-d+1)\sigma^2\textbf{I} - \sum_{\tau=t-k+d}^tY_tY_t^{\top}}\right)(w-w_t)
\end{eqnarray}
where $w_t$ is the filter that was used by the algorithm for prediction in time $t$ and $k\in{\mathbb{N}^+}$ is a parameter. \\
Our Gradient Decent filtering algorithm is given below.

\begin{algorithm}
\caption{GDFilter}
\label{GDFilter}
\begin{algorithmic}[1]
\STATE Input: $k \in{\mathbb{N}^+}$, $H\in{\mathbb{R}^+}$, $R\in{\mathbb{R}^+}$.
\STATE Let $w_1 = \textbf{0}_d$
\FOR{$c = 1...$}
\FOR{$t = (c-1)k+1...ck$}
\STATE predict: $x_t = w_c^{\top}Y_t$.
\ENDFOR
\STATE $\eta_c \leftarrow \frac{1}{Hc}$
\STATE $\tilde{w}_{c+1} \leftarrow w_c - \eta_c\nabla{}L^k_c(w_c)$.
\IF{$\Vert{\tilde{w}_{c+1}}\Vert > R$}
\STATE $w_{c+1} \leftarrow \tilde{w}_{c+1}\cdot{}\frac{R}{\Vert{\tilde{w}_{c+1}}\Vert}$.
\ELSE
\STATE $w_{c+1} \leftarrow \tilde{w}_{c+1}$.
\ENDIF
\ENDFOR
\end{algorithmic}
\end{algorithm}

We have the following theorem and corollary.
\begin{theorem}\label{GDThr}
Let $w_t$ be the filter used by algorithm \ref{GDFilter} for prediction in time $t$. Let $k$ = $2d$ and $H = d\sigma^2$. Algorithm \ref{GDFilter} achieves the following regret bound,
\begin{eqnarray*}
\mathbb{E}\left[{\sum_{t=1}^Tl_t(w_t)}\right] - \min_{w\in{\mathbb{R}^d},\Vert{w}\Vert \leq R}\mathbb{E}\left[{\sum_{t=1}^Tl_t(w)}\right] = O\left({\frac{d^3R^2(B_X+B_N)^4}{\sigma^2}\log{T}}\right) \\
\end{eqnarray*}
\end{theorem}

\begin{corr}
Let $w_t$ be the filter used by algorithm \ref{GDFilter} for prediction in time $t$. Let $k=2d$, $H = d\sigma^2$ and let $R = \frac{\sqrt{d}B_X^2}{\sigma^2}$. It holds that,
\begin{eqnarray*}
\mathbb{E}\left[{\sum_{t=1}^Tl_t(w_t)}\right] - \min_{w\in{\mathbb{R}^d}}\mathbb{E}\left[{\sum_{t=1}^Tl_t(w)}\right] = O\left({\frac{d^4B_X^4(B_X+B_N)^4}{\sigma^6}\log{T}}\right)
\end{eqnarray*}
\end{corr}

Basically the new loss function (\ref{NewLossFunc}) sums several consecutive losses and adds a regularization expression. We show that since the regularization expression depends on the actual choices of the filtering algorithm, achieving low regret with respect to $L^k_t(w)$ implies low regret with respect to the losses $l_t(w)$. Moreover, as we will show, the combination of summing several losses and adding regularization, insures that $L^k_t(w)$ is always strongly-convex for a proper choice of $k$, and thus we can use the algorithms in \cite{LogRegretPaper, AdapRegretPaper} to get logarithmic regret. \\
It holds that,
\begin{eqnarray}\label{StrongConvexInq}
\nabla^2L^k_t(w) &=& \sum_{\tau = t-k+1}^t\nabla^2\hat{l}_t(w) + 2\left({(k-d+1)\sigma^2\textbf{I} - \sum_{\tau=t-k+d}^tY_tY_t^{\top}}\right) \nonumber \\
&=& 2\sum_{\tau = t-k+1}^tY_tY_t^{\top} + 2(k-d+1)\sigma^2\textbf{I} - 2\sum_{\tau=t-k+d}^tY_tY_t^{\top} \nonumber \\
&\succeq & 2(k-d+1)\sigma^2\textbf{I}
\end{eqnarray}
Thus for $k\geq d$, $L^k_t(w)$ is always $2(k-d+1)\sigma^2$-strongly-convex and $2(k-d+1)\sigma^2/G^2$-exp-concave where $G = \max_{w,t}\Vert{\nabla{}L^k_t(w)}\Vert$.\\ 
We thus use the gradient decent algorithm in \cite{LogRegretPaper} by partitioning the iterations into disjoint blocks of length $k$ each, and our algorithm updates its filter every $d$ iterations according to the loss function $L^k_t(w)$ for $t=ck$, $c\in{\mathbb{Z}}$ and predicts using the same filter on all iterations in the same block. The value of $k$ is assumed to be a constant independent of $T$. \\
Abusing notation, we switch between $L^k_c(w)$ and $L^k_{ck}(w)$ interchangeably where we use $L^k_c(w)$ to refer to the loss on block number $c$ of length $k$. \\
The following Lemma plays a key part in our analysis.

\begin{lemma}\label{KeyLemma}
Let $\mathcal{A}$ be a filtering algorithm that updates its filter every $k$ iterations. Denote by $w_t$ the filter used for prediction on iteration $t$ and denote by $w_c$ the filter used to predict on the entire block $c$, that is on iterations $((c-1)\cdot{k}+1)...c\cdot{k}$. It holds that
\begin{eqnarray*}
\mathbb{E}\left[{\sum_{t=1}^Tl_t(w_t) - \sum_{t=1}^Tl_t(w)}\right] \leq \mathbb{E}\left[{\sum_{c=1}^{T/k}L^k_{ck}(w_c) - \sum_{c=1}^{T/k}L^k_{ck}(w)}\right] 
\end{eqnarray*}
\end{lemma}

\begin{proof}
First we assume w.l.o.g. that $T=b\cdot{k}$ for some $b\in{\mathbb{N}^+}$. Otherwise it holds that $T=b\cdot{k} + a$ where $0 < a < k$ and thus the regret on the additional $a$ iterations is a constant independent of $T$ and we can ignore it in the regret bound.\\
We now have,

\begin{eqnarray}\label{sum1}
&&\sum_{c=1}^{T/k}L^k_{ck}(w_c) - \sum_{c=1}^{T/k}L^k_{ck}(w)  \\
&=& \sum_{c=1}^{T/k}\left({\sum_{t=(c-1)k+1}^{ck}\hat{l}_t(w_c) + (w_{c}-w_{c})^{\top}\left({(k-d+1)\sigma^2\textbf{I} - \sum_{\tau=ck-k+d}^{ck}Y_{\tau}Y_{\tau}^{\top}}\right)(w_{c}-w_{c})}\right) \nonumber \\
&-& \sum_{c=1}^{T/k}\left({\sum_{t=(c-1)k+1}^{ck}\hat{l}_t(w) + (w-w_{c})^{\top}\left({(k-d+1)\sigma^2\textbf{I} - \sum_{\tau=ck-k+d}^{ck}Y_{\tau}Y_{\tau}^{\top}}\right)(w-w_{c})}\right) \nonumber \\
&=& \sum_{t=1}^T\left({\hat{l}_t(w_t) - \hat{l}_t(w)}\right) - \sum_{c=1}^{T/k}(w-w_{c})^{\top}\left({(k-d+1)\sigma^2\textbf{I} - \sum_{\tau=(c-1)k+1}^{ck}Y_{\tau}Y_{\tau}^{\top}}\right)(w-w_{c}) \nonumber
\end{eqnarray}
Since $\mathcal{A}$ updates its filter every $k$ iterations, we have that $w_{ck}$ depends only on the random variables $n_1,...,n_{(c-1)k}$. Thus using (\ref{ExpectationOfMatrix}) we have for all $c$ we that,

\begin{eqnarray*}
&&\mathbb{E}\left[{(w-w_{c})^{\top}\left({(k-d+1)\sigma^2\textbf{I} - \sum_{\tau=(c-1)k+1}^{ck}Y_{\tau}Y_{\tau}^{\top}}\right)(w-w_{c})}\right]\\
&=& (k-d+1)\sigma^2\mathbb{E}[\Vert{w-w_{c}}\Vert^2] - \mathbb{E}\left[{\sum_{\tau=(c-1)k+1}^{ck}Y_{\tau}Y_{\tau}^{\top}}\right]\circ{}\mathbb{E}\left[{(w-w_{c})(w-w_{c})^{\top}}\right]\\
&=& (k-d+1)\sigma^2\mathbb{E}[\Vert{w-w_{c}}\Vert^2] \\
&-& \left({\sum_{\tau=(c-1)k+1}^{ck}X_{\tau}X_{\tau}^{\top}+(k-d+1)\sigma^2\textbf{I}}\right)\circ{}\mathbb{E}\left[{(w-w_{c})(w-w_{c})^{\top}}\right]\\
&=& -\sum_{\tau=(c-1)k+1}^{ck}X_{\tau}X_{\tau}^{\top}\circ{}\mathbb{E}\left[{(w-w_{c})(w-w_{c})^{\top}}\right] \leq 0
\end{eqnarray*}
Overall by taking expectation over (\ref{sum1}) we get
\begin{eqnarray*}
\mathbb{E}\left[{\sum_{c=1}^{T/k}L^k_{ck}(w_c) - \sum_{c=1}^{T/k}L^k_{ck}(w)}\right] \geq \mathbb{E}\left[{\sum_{t=1}^T\hat{l}_t(w_t) - \hat{l}_t(w)}\right]
\end{eqnarray*}
The lemma now follows from (\ref{ExpectationEq}).
\end{proof}
According to Lemma \ref{KeyLemma} we can reduce our discussion to algorithms that predict in disjoint blocks of length $k$ and achieve low regret with respect to the loss function $L^k_c(w)$ \\
In order to derive precise regret bounds we give a bound on $G = \max_{w,t}\Vert{\nabla{}L^k_t(w)}\Vert$.
\begin{eqnarray*}
\nabla{}L^k_t(w) = 2\sum_{\tau=t-k+1}^tY_t(y_t - w_t^{\top}Y_t) + 2\left({(k-d+1)\sigma^2\textbf{I} - \sum_{\tau=t-k+d}^tY_{\tau}Y_{\tau}^{\top}}\right)(w-w_t)
\end{eqnarray*}
Thus by simple algebra we have,
\begin{eqnarray*}
G^2 &=& O\left({k^2d(B_X+B_N)^2R^2d(B_X+BN)^2 + k^2d^2(B_X+B_N)^4R^2}\right) \\
&=& O\left({k^2d^2R^2(B_X+B_N)^4}\right)
\end{eqnarray*}
Where $R$ is a bound on the magnitude of the filter. That is we consider only filters $w\in{\mathbb{R}^d}$ such that $\Vert{w}\Vert_2 \leq R$. $R$ needs to be bounded since the regret of online convex optimization algorithms grows with $G$.\\ 
As pointed out in \cite{MWPaper}, for 
\begin{eqnarray*}
w^* = \arg\min_{w\in{\mathbb{R}^d}}\mathbb{E}\left[{(1/T)\sum_{t=1}^T\left({x_t - w^{\top}Y_t}\right)^2}\right]
\end{eqnarray*}
It holds that $\Vert{w^*}\Vert \leq \frac{\sqrt{d}B_X^2}{\sigma^2}$. \\
We denote by $G(k,R)$ an upper bound on $\max_{w,t}\Vert{\nabla{}L^k_t(w)}\Vert$ parametrized by $k,R$. \\
For the complete proof of the theorem and corollary the reader is referred to the appendix.

\section{An Adaptive Algorithm}
In this section we present an algorithm that is based on the framework from \cite{AdapRegretPaper} and achieves logarithmic expected regret on any interval $I=[r,s]\subseteq{[T]}$. Our algorithm is given below.
\begin{algorithm}
\caption{AdaptiveFilter}
\label{AdaptiveFilter}
\begin{algorithmic}[1]
\STATE Input: $k\in{\mathbb{N}^+}$, $\alpha\in{\mathbb{R}^+}$.
\STATE Let $E^1,...,E^T$ be online convex optimization algorithms.
\STATE Let $p_1\in{\mathbb{R}^T},p_1^{(1)} = 1, \forall{j:1<j\leq T}, p_1^{(j)} = 0$.
\FOR{$c = 1...$}
\STATE $\forall{j \leq c}, w_c^{(j)} \leftarrow E^j(L^k_1,...,L^k_{(c-1)})$ (the filter of the j'th algorithm).
\STATE $w_c \leftarrow \sum_{j=1}^cp_c^{(j)}w_c^{(j)}$.
\FOR{$t = (c-1)k+1...ck$}
\STATE predict: $x_t = w_c^{\top}Y_t$.
\ENDFOR
\STATE $\hat{p}_{c+1}^{(c+1)} = 0$ and for $i\in{[c]}$, \begin{eqnarray*}
\hat{p}_{c+1}^{(i)} = \frac{p_c^{(i)}e^{-\alpha{}L^k_c(w_c^{(i)})}}{\sum_{j=1}^cp_c^{(i)}e^{-\alpha{}L^k_c(w_c^{(i)})}}
\end{eqnarray*}
\STATE $p_{c+1}^{(c+1)} = 1/(c+1)$ and for $i\in{[c]}: p_{c+1}^{(i)} = (1-(c+1)^{-1})\hat{p}_{c+1}^{(i)}$ (adding expert $E^{(c+1)}$).
\ENDFOR
\end{algorithmic}
\end{algorithm}
We have the following theorem and corollary.
\begin{theorem}\label{AdaptiveThr}
Let $w_t$ be the filter used by algorithm \ref{AdaptiveFilter} for prediction in time $t$. Let $k=2d$ and let $\alpha = \frac{d\sigma^2}{G(2d,R)^2}$. For all $I=[r,s]\subseteq{[T]}$, algorithm \ref{AdaptiveFilter} achieves the following regret bound,
\begin{eqnarray*}
\mathbb{E}\left[{\sum_{t=r}^sl_t(w_t)}\right] - \min_{w\in{\mathbb{R}^d},\Vert{w}\Vert \leq R}\mathbb{E}\left[{\sum_{t=r}^sl_t(w)}\right] = O\left({\frac{d^3R^2(B_X+B_N)^4}{\sigma^2}\log{T}}\right)\\
\end{eqnarray*} 
\end{theorem}

\begin{corr}
Let $w_t$ be the filter used by algorithm \ref{AdaptiveFilter} for prediction in time $t$. Let $k=2d$, $R = \frac{\sqrt{d}B_X^2}{\sigma^2}$ and let $\alpha = \frac{d\sigma^2}{G(2d,R)^2}$. For all $I=[r,s]\subseteq{[T]}$, algorithm \ref{AdaptiveFilter} achieves the following regret bound,
\begin{eqnarray*}
\mathbb{E}\left[{\sum_{t=r}^sl_t(w_t)}\right] - \min_{w\in{\mathbb{R}^d}}\mathbb{E}\left[{\sum_{t=r}^sl_t(w)}\right] = O\left({\frac{d^4B_X^4(B_X+B_N)^4}{\sigma^6}\log{T}}\right)
\end{eqnarray*}
\end{corr}

As in the previous section, we take the approach of partitioning the iterations into disjoint blocks of length $k$ and optimizing over the loss functions $L^k_t$. \\
The algorithm is based on the well known experts framework where each expert in our case, is a gradient descent filter presented in the previous section. On each block $c$, the algorithm adds a new expert that starts producing predictions from block $c+1$ an onward. The experts algorithm predicts on each iteration by combining the filters of all experts using a weighted sum according to the weight of each expert. The key idea behind this framework is that an expert added at block $r$ achieves low regret on all intervals starting in $r$. Given such an interval, the experts algorithm itself achieves low regret on the interval with respect to this specific expert, and thus has low regret on the interval. \\
Expert $E^r$ could be thought of as an algorithm that plays $w_c = 0$ for all $c<r$ and starting at block $r$ plays according to algorithm \ref{GDFilter}. \\
For the complete proof of the theorem and corollary the reader is referred to the appendix.

\bibliography{bib}{}

\begin{thebibliography}{10}

\bibitem{MWPaper}
Taesup Moon and Tsachy Weissman.
\newblock Universal fir mmse filtering.
\newblock {\em IEEE Transactions on Signal Processing}, 57(3):1068--1083, 2009.

\bibitem{WienerPaper}
Norbert Wiener.
\newblock {\em Extrapolation, Interpolation, and Smoothing of Stationary Time
  Series, With Engineering Applications}.
\newblock New York: Wiley, 1949.

\bibitem{LinearEstimationRef}
Thomas Kailath, Ali~H. Sayed, and Babak Hassibi.
\newblock {\em Linear Estimation}.

\bibitem{RobustMMRef1}
H.~Vincent Poor.
\newblock On robust wiener filtering.
\newblock {\em IEEE Trans. Automatic Control}, AC-25:521--526, Jun 1980.

\bibitem{RobustMMRef2}
Yonina~C. Eldar and Neri Merhav.
\newblock A competitive minimax approach to robust estimation of random
  parameters.
\newblock {\em IEEE Trans. Signal Processing}, 52:1931--1946, July 2004.

\bibitem{RobustMMRef3}
Yonina~C. Eldar, Aharon Ben-Tal, and Arkadi Nemirovski.
\newblock Linear minimax regret estimation of deterministic parameters with
  bounded data uncertainties.
\newblock {\em IEEE Trans. Signal Processing}, 52:2177--2188, Aug 2004.

\bibitem{UnsuperAdapFilterRef}
Simon Haykin.
\newblock {\em Unsupervised Adaptive Filtering: Volume I, II}.
\newblock New York:Wiley, 2000.

\bibitem{ZinkPaper}
Martin Zinkevich.
\newblock Online convex programming and generalized infinitesimal gradient
  ascent.
\newblock In {\em ICML}, pages 928--936, 2003.

\bibitem{LogRegretPaper}
Elad Hazan, Amit Agarwal, and Satyen Kale.
\newblock Logarithmic regret algorithms for online convex optimization.
\newblock {\em Machine Learning}, 69(2-3):169--192, 2007.

\bibitem{AdapRegretPaper}
Elad Hazan and C.~Seshadhri.
\newblock Efficient learning algorithms for changing environments.
\newblock In {\em ICML}, page~50, 2009.

\end{thebibliography}
\bibliographystyle{unsrt}

\appendix

\section{Proof of Theorems \ref{GDThr}, \ref{AdaptiveThr}}
The proofs are based on \cite{LogRegretPaper, AdapRegretPaper} and are brought here in full detail for completeness.
\begin{theorem}
Let $w_t$ be the filter used by algorithm \ref{GDFilter} for prediction in time $t$. Let $k$ = $2d$ and $H = d\sigma^2$. Algorithm \ref{GDFilter} achieves the following regret bound,
\begin{eqnarray*}
\mathbb{E}\left[{\sum_{t=1}^Tl_t(w_t)}\right] - \min_{w\in{\mathbb{R}^d},\Vert{w}\Vert \leq R}\mathbb{E}\left[{\sum_{t=1}^Tl_t(w)}\right] = O\left({\frac{d^3R^2(B_X+B_N)^4}{\sigma^2}\log{T}}\right)
\end{eqnarray*}
\end{theorem}

\begin{proof}
Again we assume w.l.o.g that $T = b\cdot{k}$ for some $b\in{\mathbb{N}^+}$.
Consider some $w\in{\mathbb{R}^d}$ such that $\Vert{w}\Vert_2 \leq R$. Define $\nabla_{c} = \nabla{}L^k_{c}(w_c)$ and $\nabla_c^2 = \nabla^2L^k_{c}(w_c)$, $G = G(2d,R)$. Writing the Taylor series approximation of $L^k_c(w)$ around $w_c$ we have,
\begin{eqnarray*}
L^k_c(w) &=& L^k_c(w_c) + \nabla_c^{\top}(w - w_c) + \frac{1}{2}\nabla_c^2 \circ{} (w - w_c)(w - w_c)^{\top} 
\end{eqnarray*}
According to (\ref{StrongConvexInq}), $\nabla_c^2 \succeq 2(k-d+1)\sigma^2\textbf{I}$ and we have,
\begin{eqnarray}\label{GDTaylor}
L^k_c(w) \geq L^k_c(w_c) + \nabla_c^{\top}(w - w_c) + (k-d+1)\sigma^2\Vert{w - w_c}\Vert_2^2
\end{eqnarray}

Following the analysis in \cite{ZinkPaper, LogRegretPaper} we upper bound $\nabla_c^{\top}(w - w_c)$ by,

\begin{eqnarray}\label{GDbound}
2\nabla_c^{\top}(w - w_c) \leq \frac{\Vert{w_c - w}\Vert^2 - \Vert{w_{c+1} - w}\Vert^2}{\eta_{c+1}} + \eta_{c+1}G^2
\end{eqnarray}
Summing over (\ref{GDbound}) for all $c$, using (\ref{GDTaylor}) we have,

\begin{eqnarray*}
2\sum_{c=1}^{T/k}L^k_c(w_c) - L^k_c(w)
&\leq & \sum_{c=1}^{T/k}\Vert{w_c-w}\Vert^2\left({H(c+1) - Hc - (k-d+1)\sigma^2}\right) \\
&+& G^2\sum_{c=1}^{T/k}\frac{1}{Hc}
\end{eqnarray*}
Plugging $H = d\sigma^2$ yields
\begin{eqnarray*}
\sum_{t=c}^{T/k}L^k_c(w_c) - L^k_c(w) = O\left({\frac{G^2}{d\sigma^2}\log{T}}\right)
\end{eqnarray*}
The theorem now follows from (\ref{KeyLemma}) and plugging $G=G(2d,R)$.
\end{proof}

In order to prove Theorem \ref{AdaptiveThr} we need two simple claims first. In what follows we assume that $L^k_c(w)$ is $\alpha$-exp-concave.

\begin{claim}
1. For $i < c$, 
\begin{eqnarray*}
L^k_{c}(w_c) - L^k_{c}(w_c^{(i)}) \leq \alpha^{-1}
\end{eqnarray*}
2. $L^k_c(w_c) - L^k_c(w_c^{(c)} \leq \alpha^{-1}(\ln{\hat{p}_{c+1}^{(c)}} + \ln{c})$
\end{claim}

\begin{proof}
Using the $\alpha$-exp concavity of $L^k_c$ we have
\begin{eqnarray*}
e^{-\alpha{}L^k_c(w_c)} = e^{-\alpha{}L^k_c(\sum_{j=1}^cp_c^{(j)}x_c^{(j)}} 
\geq \sum_{j=1}^cp_c^{(j)}e^{-\alpha{}L^k_c(x_c^{(j)}}
\end{eqnarray*}
Taking logarithm,
\begin{eqnarray*}
L^k_c(w_c) \leq \alpha^{-1}\ln{\sum_{j=1}^cp_c^{(j)}e^{-\alpha{}L^k_c(w_c^{(j)})}}
\end{eqnarray*}
Thus,
\begin{eqnarray}\label{AdaptiveClaimEq}
L^k_c(w_c) - L^k_c(w_c^{(i)}) \nonumber \\
&\leq & \alpha^{-1}\left({\ln{e^{-\alpha{}L^k_c(w_c^{(i)})}} - \ln{\sum_{j=1}^cp_c^{(j)}e^{-\alpha{}L^k_c(w_c^{(j)})}}}\right) \nonumber \\
&=& \alpha^{-1}\ln{\frac{e^{-\alpha{}L^k_c(w_c^{(i)})}}{\sum_{j=1}^cp_c^{(j)}e^{-\alpha{}L^k_c(w_c^{(j)})}}} \nonumber \\
&=& \alpha^{-1}\ln{\left({\frac{1}{p_c^{(i)}} \cdot\frac{p_c^{(i)}e^{-\alpha{}L^k_c(w_c^{(i)})}}{\sum_{j=1}^cp_c^{(j)}e^{-\alpha{}L^k_c(w_c^{(j)})}}}\right)} \nonumber \\
&=& \alpha^{-1}\ln{\frac{\hat{p}_{c+1}^{(i)}}{p_c^{(i)}}}
\end{eqnarray}
Now, by definition it holds that for $i < c$, $p_c^{(i)} = (1-1/c)\hat{p}_c^{(i)}$. Also, $p_c^{(c)} = 1/c$. Plugging these two equalities into (\ref{AdaptiveClaimEq}) yields the claim.  
\end{proof}

\begin{claim}\label{AdaptiveMainClaim}
For any two integers $r,s$ such that $s>r$, it holds that
\begin{eqnarray*}
\sum_{c=r}^sL^k_c(w_c) - L^k_c(w_c^{(r)}) \leq \frac{4}{\alpha}\ln{T}
\end{eqnarray*}
\end{claim}

\begin{proof}
Using the previous claim we have,
\begin{eqnarray*}
&&\sum_{c=r}^sL^k_c(w_c) - L^k_c(w_c^{(r)}) \\
&=& (L^k_r(w_r) - L^k_r(w_r^{(r)})) + \sum_{c=r+1}^sL^k_c(w_c) - L^k_c(w_c^{(r)}) \\
&\leq & \alpha^{-1}\left({\ln{\hat{p}_{r+1}^{(r)}} + \ln{r} + \sum_{c=r+1}^s\ln{\hat{p}_{c+1}^{(r)}} -\ln{\hat{p}_{c}^{(r)}} + 2/c}\right) \\
&=& \alpha^{-1}\left({\ln{r} + \ln{\hat{p}_{s+1}^{(r)}} + \sum_{c=r+1}^s2/c}\right)
\end{eqnarray*}
Since $\hat{p}_{s+1}^{(r)} \leq 1$, $\ln{\hat{p}_{s+1}^{(r)}} \leq 0$. This implies that the regret is bounded by $\frac{4}{\alpha}\ln{T}$.
\end{proof}
We can now prove Theorem \ref{AdaptiveThr}.

\begin{theorem}
Let $w_t$ be the filter used by algorithm \ref{AdaptiveFilter} for prediction in time $t$. Let $k=2d$ and let $\alpha = \frac{d\sigma^2}{G(2d,R)^2}$. For all $I=[r,s]\subseteq{[T]}$, algorithm \ref{AdaptiveFilter} achieves the following regret bound,
\begin{eqnarray*}
\mathbb{E}\left[{\sum_{t=r}^sl_t(w_t)}\right] - \min_{w\in{\mathbb{R}^d},\Vert{w}\Vert \leq R}\mathbb{E}\left[{\sum_{t=r}^sl_t(w)}\right] = O\left({\frac{d^3R^2(B_X+B_N)^4}{\sigma^2}\log{T}}\right)
\end{eqnarray*}
\end{theorem}

\begin{proof}
Given an interval $I=[r,s]\subseteq{[T]}$, let $r = c_r\cdot{}k - b_r$, $s = c_s\cdot{}k + b_s$ such that $c_r,b_r,c_s,b_s\in{\mathbb{N}}$ and $0 \leq b_r, b_s \leq k-1$. \\
Since $k$ is a constant independent of $T$, we ignore the first $b_r$ iterations and last $b_s$ iterations, since they only add a constant to the regret. \\
According to Claim \ref{AdaptiveMainClaim} we have,
\begin{eqnarray*}
\sum_{c=c_r}^{c_s}L^k_c(w_c) - L^k_c(w_c^{(r)}) \leq \frac{4}{\alpha}\ln{T} = O\left({\frac{G(2d,R)^2}{d\sigma^2}\log{T}}\right)
\end{eqnarray*}
Since $E^r$ achieves low regret on all block-intervals beginning in block $r$ we have for all $w\in{\mathbb{R}}$ such that $\Vert{w}\Vert_2 \leq R$,
\end{proof}
\begin{eqnarray*}
\sum_{c=c_r}^{c_s}L^k_c(w_c^{(r)}) - L^k_c(w) = O\left({\frac{G(2d,R)^2}{d\sigma^2}\log{T}}\right)
\end{eqnarray*}
Thus we have,
\begin{eqnarray*}
\sum_{c=c_r}^{c_s}L^k_c(w_c) - L^k_c(w) = O\left({\frac{G(2d,R)^2}{d\sigma^2}\log{T}}\right)
\end{eqnarray*}
Again, the theorem now follows from (\ref{KeyLemma}) and plugging $G=G(2d,R)$.
\end{document}